\pdfoutput=1
\documentclass{article}

\usepackage{arxiv}
\usepackage[numbers]{natbib}

\usepackage[utf8]{inputenc} 
\usepackage[T1]{fontenc}    
\usepackage{hyperref}       
\usepackage{url}            
\usepackage{booktabs}       
\usepackage{amsfonts}       
\usepackage{nicefrac}       
\usepackage{microtype}      
\usepackage{lipsum}
\usepackage{graphicx}
\usepackage{pgfplots}
\usepackage{diagbox}
\usepackage{amsmath}
\usepackage[ruled,vlined]{algorithm2e}
\graphicspath{ {./images/} }

\usepackage{booktabs} 
\usepackage{tikz} 
\usepackage{amsthm} 
\usepackage{float}
\usepackage{amsfonts}
\usepackage{enumitem}
\usepackage{amssymb}
\usepgfplotslibrary{external} 
\usepackage{filecontents}

\usepackage{subcaption}
\pgfplotscreateplotcyclelist{mycolorlist}{%
    {cyan, mark=o},%
    {magenta, mark=square},%
    {green!60!black, mark=triangle},%
    {purple, mark=diamond},%
    {orange, mark=star}%
}

\usepackage{color}
\newtheorem{proposition}{Proposition}

\newtheorem{lemma}{Lemma}

\newtheorem{theorem}{Theorem}

\newtheorem{proof sketch}{Proof Sketch}
\newtheorem{assumption}{Assumption}

\title{Fairness Auditing: Lower Bounds on Company Manipulation}

\author{
 Rachit Verma \\
  Indian Institute of Technology Gandhinagar\\
   \And
 Padala Manisha  \\
  Indian Institute of Technology Gandhinagar\\
  \And
 Sujit Gujar \\
  International Institute of Information Technology Hyderabad\\
}

\begin{document}
\maketitle
\begin{abstract}
Fairness audits are increasingly mandated in high-stakes applications such as hiring, lending, and automated decision-making. Recent work has established fundamental impossibility results for black-box fairness auditing, showing that sufficiently expressive models can evade any auditing strategy. We complement these results by quantifying the extent of unavoidable post-audit manipulation under finite audit resources.
We formulate fairness auditing as a min-max optimization between a computationally unbounded company and a budget-constrained auditor. We study two auditing regimes: (i) a budgeted auditor that certifies fairness using a fixed-size audit set, and (ii) a budgeted $\alpha$-tolerant auditor that additionally requires the audit set to estimate the fairness of the certified model within an $\alpha$ approximation. For both settings, we derive explicit lower bounds on the worst-case post-audit demographic parity deviation as functions of the audit budget, group imbalance, and fairness tolerance.
Finally, we empirically illustrate these theoretical limits using simple audit-set construction heuristics with linear and neural network classifiers. Our results demonstrate that increasing audit resources reduces, but does not eliminate, the scope for post-audit manipulation, highlighting fundamental limitations of finite-budget fairness certification. 
\end{abstract}



\section{Introduction} Machine learning (ML) models are increasingly deployed in high-stakes applications such as hiring, lending, and criminal justice, where concerns about algorithmic fairness have motivated regulations requiring bias assessments before deployment \cite{angwin2016machine,fabris2025fairness,barocas16,nyc_local_law_144,eu_ai_act_2024}. Consequently, fairness auditing has emerged as an important mechanism for certifying that deployed models satisfy prescribed fairness criteria.

A fundamental challenge, however, is that fairness audits are inherently resource constrained. Auditors typically certify a model using only a limited number of queries or samples, while the deployed model itself remains inaccessible. This creates opportunities for a strategic company to present a model that passes the audit while deploying a different model that agrees with the audited model only on the inspected samples (Figure~\ref{fig:classifier_deviation}).

Recent work has formalized this setting in the black-box auditing framework \cite{yan2022active,Godinot_2024,bourree2025robust}. In particular, \cite{Godinot_2024} proves that for sufficiently expressive hypothesis classes, no black-box auditing strategy can outperform uniform random sampling. These results establish that strategic manipulation is fundamentally unavoidable. However, they leave open a complementary quantitative question:

\begin{quote}
\emph{Given a finite audit budget, how much post-audit unfairness can a strategic company always guarantee?}
\end{quote}

\begin{figure*}[!t]
    \centering
    \begin{minipage}[b]{0.30\textwidth}
        \centering
        \scalebox{0.70}{%
            \begin{tikzpicture}

    \draw[->] (-0.5,0) -- (5,0) node[right] {Feature $1$};
    \draw[->] (0,-0.5) -- (0,5) node[above] {Feature $2$};

    \draw[thick] (1,-0.5) -- (4,5) node[right] {\textbf{$h$}};
    \draw[dashed] (0, 0) -- (5.2,5) node[right] {\textbf{$h^\prime$}};
    \draw[dashed] (0.2,2) -- (6,4) node[right] {\textbf{$h^{''}$}};

    \foreach \x/\y in {3/4, 2.5/3.5, 2.7/3.2, 2.3/4.2} {
        \filldraw[red] (\x,\y) circle(4pt);
    }

    \draw[red, thick] (4.8,4) circle(4pt);

    \draw[red, thick] (3.8,4) circle(4pt);

    \foreach \x/\y in {2.3/1, 2.5/1.2, 2.8/1.3, 3.1/1.1} {
        \draw[blue] (\x,\y) node {\large$\triangle$};
    }

    \node[blue] at (0.8,1.3) {\large$\blacktriangle$};
    
    \node[blue] at (1.7,1.2) {\large$\blacktriangle$};

    \draw[yellow, line width=1mm](0.6,1.12) -- (1,1.12) -- (0.8,1.57) -- cycle;

    \draw[yellow, line width=1mm] (4.8,4) circle(6pt);

\end{tikzpicture}%
        }%
    \end{minipage}
    \hfill
    \begin{minipage}[b]{0.65\textwidth}
        \small
        Consider the set of points $E$. The red circles belong to the sensitive group $0$ and the blue triangles belong to the sensitive group $1$. In this case, the company starts with the classifier $h$, and the auditor chooses the points highlighted in yellow as the auditing set. Points that are visually filled with color denote instances assigned a positive label, whereas points represented by outlines indicate instances assigned a negative label. Clearly, $\Delta(E,h) = |\frac{4}{6} - \frac{2}{6}| = \frac{1}{3}$. The company then changes its classifier to $h^\prime$, with $\Delta(E,h^\prime) = |\frac{5}{6} - \frac{1}{6}| = \frac{2}{3}$ while being consistent with the audit set, achieving its goal of deviation. The company cannot deviate to $h^{''}$ since it will violate consistency with the auditing set.
    \end{minipage}
    
    \caption{Analysis of classifier deviations under auditing sets.}
    \label{fig:classifier_deviation}
\end{figure*}

In this paper, we study this question by modeling fairness auditing as a min-max optimization between a computationally unbounded company and a budget-constrained auditor. Rather than proposing a new auditing algorithm, we characterize the limits of fairness certification under finite audit resources. We consider two natural auditing regimes. In the first, the auditor is restricted only by the size of the audit set. In the second, the auditor must additionally ensure that the certified fairness estimate approximates the fairness of the audited model within an additive tolerance~$\alpha$.

Our main contribution is an explicit characterization of the unavoidable post-audit demographic parity deviation under these audit constraints. For a budget-constrained auditor, Theorem~\ref{thm:theorem_1} derives a lower bound on the post-audit demographic parity deviation as a function of the audit budget and group imbalance. We further extend this analysis to an $\alpha$-tolerant auditor in Theorem~\ref{thm:theorem_2}, where the auditor must additionally certify the fairness of the audited model within an additive tolerance. Together, these results complement existing impossibility results by making the dependence on finite audit resources explicit and providing quantitative guarantees on what fairness audits can certify.

Finally, we empirically illustrate these theoretical limits using linear and neural network classifiers. We instantiate the audit constraints using a simple audit-set construction heuristic (Algorithm~\ref{alg:rasc}) and compare against random auditing. The experiments are intended to illustrate the theoretical trends rather than introduce a new auditing algorithm. Consistent with our analysis, increasing the audit budget and tightening the fairness tolerance reduce, but do not eliminate, the scope for post-audit manipulation.

\noindent\textbf{Related Work.} Strategic manipulation in fairness auditing was introduced by \cite{yan2022active}, who formalized \emph{manipulation-proof auditing} and showed that a company can present a compliant model during an audit while deploying a different model afterward. Subsequent work showed that private auditor priors can mitigate such manipulation under certain conditions, whereas public priors remain vulnerable \cite{bourree2025robust}.

Most closely related to our work, \cite{Godinot_2024} established fundamental impossibility results for black-box fairness auditing, proving that for sufficiently expressive hypothesis classes no auditing strategy can outperform uniform random sampling. More broadly, limited access has also been identified as a fundamental obstacle in AI safety evaluations \cite{Casper_2024}. Our work complements these results by quantifying the extent of unavoidable post-audit manipulation under finite audit budgets, providing explicit lower bounds that depend on the audit budget, group imbalance, and fairness tolerance.

Several other works study complementary forms of audit manipulation, including biased sampling \cite{fukuchi2020faking}, statistical fairwashing \cite{shahin2022washing}, auditing with external datasets \cite{betting}, and practical challenges in real-world bias audits \cite{auditing_audits,pragmatic_fairness}. In contrast, we focus on the strategic manipulation of the deployed model itself and characterize the limits of fairness certification under finite audit resources.

\section{Preliminaries}
In this section, we define the major notations and definitions formally.

\subsection{Classification Setup} 
Let \(D=\{(x_i,a_i,y_i)\}_{i=1}^n\) be a dataset, where \(x_i\in\mathcal{X}\subset\mathbb{R}^d\) denotes the feature vector, \(a_i\in\{0,1\}\) the binary sensitive attribute, and \(y_i\in\{0,1\}\) the label. A binary classifier \(h:\mathcal{X}\rightarrow\{0,1\}\) predicts \(\hat y_i=h(x_i)\).

\smallskip
\noindent\textbf{Demographic Parity.}
We measure fairness using the empirical \emph{Demographic Parity} (DP) violation~\cite{dwork12}. For any subset of samples \(S\subseteq D\), the DP violation of a classifier \(h\) is defined as

\begin{equation}
\Delta(S,h)=
\left|
\frac{\sum_{(x_i,a_i)\in S} h(x_i)a_i}
{\sum_{a_i\in S} a_i}
-
\frac{\sum_{(x_i,a_i)\in S} h(x_i)(1-a_i)}
{\sum_{a_i\in S}(1-a_i)}
\right|.
\label{eq:dp}
\end{equation}

A classifier satisfies demographic parity on \(S\) if \(\Delta(S,h)=0\). When the subset is the entire dataset, we simply write
\[
\Delta(h) := \Delta(D,h).
\]

\subsection{Auditing Setup}

A company owns a proprietary classifier \(h^\star\) that must be certified before deployment by an external auditor. The auditor selects an audit set \(S\subseteq D\) of size at most \(B\), observes the predictions of \(h^\star\) on \(S\), and estimates its demographic parity violation \(\Delta(S,h^\star)\).

After certification, the company may deploy a classifier \(h'\). To remain consistent with the audit, the deployed classifier must satisfy

\[
h'(x)=h^\star(x), \qquad \forall x\in S,
\]
while it may differ arbitrarily outside the audit set (Figure~\ref{fig:classifier_deviation}).

Following the worst-case analysis of prior work \cite{Godinot_2024}, we assume the company has unrestricted representational capacity and may choose any classifier satisfying the above consistency constraint. This assumption is motivated by the ability of over-parameterized models to memorize arbitrary labelings \cite{zhang2021understanding}.

\paragraph{Group-wise notation.}

For the audited classifier \(h^\star\), let \[
X_{ij}
=
\{(x,a,y)\in D:\;a=i,\;h^\star(x)=j\},
\] where \(i,j\in\{0,1\}\), and denote
\(x_{ij}=|X_{ij}|\).

The corresponding group sizes are \[
s_0=x_{00}+x_{01},
\qquad
s_1=x_{10}+x_{11},
\] and, without loss of generality, we assume group \(0\) is privileged. Throughout the paper, \(
N=x_{00}+x_{11}
\) denotes the total number of samples whose labels can increase the demographic parity violation if manipulated. Under this convention,

\[
\Delta(h^\star)
=
\frac{x_{01}}{s_0}
-
\frac{x_{11}}{s_1},
\]
and we write \(
s_{\max}=\max(s_0,s_1).
\)

\section{Formulation for Worst-case Analysis}
\label{sec:problem_formulation}
We formulate fairness auditing as a two-player min-max optimization between a budget-constrained auditor and a strategic company. The auditor first selects an audit set subject to its certification constraints, after which the company deploys a classifier consistent with the audit while maximizing demographic parity (DP) violation. The auditor operates under one of the following audit models.

\begin{itemize}[leftmargin=*]

\item \textbf{Budgeted Audit.}
The auditor is restricted only by the audit budget,
\begin{equation}
\mathcal{S}_B
=
\{S\subseteq D:\; |S|\le B\}.
\label{eq:only_bug}
\end{equation}

\item \textbf{Budgeted $\alpha$-Tolerant Audit.}
In addition to the budget constraint, the audit set must estimate the fairness of the certified classifier within an additive tolerance $\alpha$,
\begin{equation}
\mathcal{S}_{B,\alpha}
=
\left\{
S\subseteq D:
|S|\le B,\;
|\Delta(S,h^\star)-\Delta(h^\star)|\le\alpha
\right\}.
\label{eq:bud_alpha_tol}
\end{equation}

\end{itemize}

\vspace{1mm}

\noindent\textbf{Feasible Classifiers.}
For a certified classifier \(h^\star\) and audit set \(S\), define

\[
\mathcal H_{S,h^\star}
=
\{h\in\mathcal H:\;
h(x)=h^\star(x),\ \forall x\in S\},
\] the set of classifiers that are indistinguishable from \(h^\star\) on the audited samples.

\vspace{1mm}

\noindent\textbf{Company Objective.}
The company seeks to maximize post-audit demographic parity violation. Its utility from deploying \(h\in\mathcal H_{S,h^\star}\) is

\begin{equation}
U(h,h^\star)
=
\Delta(h)-\Delta(h^\star).
\label{eq:company_util}
\end{equation}

This objective models a worst-case strategic or adversarial company and is used to derive lower bounds on unavoidable post-audit manipulation.

The auditor minimizes this utility by selecting an audit set, while the company maximizes it by choosing a feasible classifier.

For a budgeted audit, the worst-case deviation is

\begin{equation}
\mathrm{WCD}_B(h^\star)
=
\inf_{S\in\mathcal S_B}
\;
\sup_{h\in\mathcal H_{S,h^\star}}
\left(
\Delta(h)-\Delta(h^\star)
\right),
\label{eq:wcdb}
\end{equation}

whereas for a budgeted $\alpha$-tolerant audit,

\begin{equation}
\mathrm{WCD}_{B,\alpha}(h^\star)
=
\inf_{S\in\mathcal S_{B,\alpha}}
\;
\sup_{h\in\mathcal H_{S,h^\star}}
\left(
\Delta(h)-\Delta(h^\star)
\right).
\label{eq:wcdr}
\end{equation}

In the next section, we derive explicit lower bounds for both formulations, quantifying how post-audit manipulation depends on the audit budget, group imbalance, and fairness tolerance.

\section{Quantifying the Limits of Budget-Constrained Fairness Audits}
\label{sec:lba}

In the previous section, we formulated fairness auditing as two min–max optimization problems. In this section, we derive explicit lower bounds on the values of these optimization problems. These bounds characterize the minimum post-audit demographic parity deviation that every auditor must tolerate under the respective audit constraints.

For clarity, complete proofs are deferred to Appendix~\ref{app:proofs}. Here we emphasize the intuition underlying each result.

\subsection{Budgeted Audits}

We first consider the simplest auditing model in which the auditor is constrained only by its sampling budget. Since every audited sample is protected from post-audit manipulation, increasing the budget directly limits the company's ability to alter the deployed classifier. The following theorem quantifies this dependence.

\smallskip
\noindent\textbf{Manipulation Strategy.} The lower bounds are derived by analyzing the company's most advantageous post-audit manipulation. Observe that increasing the demographic parity violation requires changing predictions in opposite directions across the two sensitive groups. Consequently, the company only needs to consider two sets of samples:

\begin{itemize}[leftmargin=1.2em,noitemsep]
    \item[] \(U\): samples from \(X_{00}\) whose predictions are changed from \(0\) to \(1\),
    \item[] \(P\): samples from \(X_{11}\) whose predictions are changed from \(1\) to \(0\).
\end{itemize}

Let \(u=|U|\) and \(p=|P|\) denote the number of manipulated samples in each set. If the auditor includes \(n_u\) samples from \(U\) and \(n_p\) samples from \(P\) in the audit set, those manipulations become infeasible. Consequently, the company's remaining post-audit deviation depends only on the unaudited manipulated samples, namely

\[
\Delta(h')-\Delta(h^\star)
=
\frac{u-n_u}{s_0}
+
\frac{p-n_p}{s_1}.
\]

The proofs of Theorems~\ref{thm:theorem_1} and~\ref{thm:theorem_2} characterize how the auditor should allocate its limited budget between these two manipulation sets. The following theorem lower bounds the value of $\mathrm{WCD}_{B}(h^\star)$ (Equation \ref{eq:wcdb}). \begin{theorem}[\textbf{Lower Bound under Budgeted Audit}] \label{thm:theorem_1}
 Let \( h^\star \) be the initial classifier, and the auditor selects an audit set \( S \subseteq D \) s.t. \( |S| \leq B \). Then, the maximum increase in demographic parity violation that the company can achieve through strategic deviation is lower bounded as:
\[
\Delta(h') - \Delta(h^\star) \geq \max\left\{ \frac{N - B}{s_{\max}},\ 0 \right\},
\]
where $h'$ is the classifier that the company ends with. Moreover, this bound is tight when \( s_0 = s_1 \).
\end{theorem}

\paragraph{Interpretation.} The lower bound admits a simple interpretation. The quantity $N=x_{00}+x_{11}$ represents the total number of samples whose labels the company would ideally manipulate in the absence of auditing.

Since the auditor can monitor at most $B$ samples, at least $N-B$ candidate manipulations remain available. Their contribution to demographic parity is normalized by the larger sensitive group, giving the factor $s_{\max}$.

Consequently, it captures the minimum manipulation that any budget-constrained auditor must tolerate. The bound is tight when the two sensitive groups are balanced.

\paragraph{Proof sketch.}

The company first proposes a manipulation set by changing labels in $X_{00}$ and $X_{11}$ to maximize demographic parity violation. The auditor responds by selecting $B$ samples from these candidate manipulations.

When $s_0<s_1$, each monitored sample from group $0$ blocks a larger increase in demographic parity than monitoring a sample from group $1$, and vice versa. Accordingly, the auditor allocates its budget to the group with the larger marginal influence until either the budget or the manipulation set is exhausted.

Analyzing these cases yields the auditor's optimal strategy. Substituting the company's optimal manipulation
\[
u=x_{00},
\qquad
p=x_{11},
\]
into the resulting expression gives the stated lower bound.

\subsection{Budgeted $\alpha$-Tolerant Audits}

The previous result allows the auditor to construct arbitrary audit sets. In practice, however, audit certificates are expected to accurately reflect the fairness of the certified model. We therefore consider auditors whose audit sets must estimate the demographic parity of the certified classifier within an additive tolerance $\alpha$. The following theorem lower bounds $\mathrm{WCD}_{B,\alpha}(h^\star)$ (Equation~\ref{eq:wcdr}).

\begin{theorem}[\textbf{Lower Bound under Budgeted $\alpha$-Tolerant Audit}]
\label{thm:theorem_2}
Let \( h^\star \) be the initial classifier with DP violation \( \Delta(h^\star) \). 
The auditor has budget \( B \) and tolerates up to \( \alpha \) deviation from the fairness of \( h^\star \). Define \( t = \Delta(h^\star) + \alpha \). Then, the maximum increase in fairness violation achievable by the company is lower bounded by:
\[
\Delta(h') - \Delta(h^\star) \geq \frac{N - B}{s_{\max}} +  \frac{x_{10} (s_1 - s_0t)}{s_0 s_1 (1 - t)},
\]
subject to feasibility conditions on \( B \), $t$, \( x_{10} \). The bound is tight when \( s_0 = s_1 \). $h'$ is the classifier that the company ends with. 
\end{theorem}

\paragraph{Interpretation.}

Unlike Theorem~\ref{thm:theorem_1}, the lower bound now consists of two components,

\[
\underbrace{\frac{N-B}{s_{\max}}}_{\text{Budget}}
\;+\;
\underbrace{
\frac{x_{10}(s_1-s_0t)}
{s_0s_1(1-t)}
}_{\text{Tolerance}}.
\]


The first term is identical to the budget-only setting and captures the unavoidable manipulation arising from unaudited samples. The second term is a consequence of the tolerance constraint. Since the auditor must preserve the demographic parity of the certified model, it loses the freedom to concentrate its entire budget on the company's manipulation set. This additional restriction creates further opportunities for strategic deviation.


\paragraph{Proof sketch.}

The $\alpha$-tolerance constraint introduces an additional coupling between the composition of the audit set and its estimated demographic parity.

We first show that an optimal auditor prioritizes monitoring manipulated samples, using non-manipulated samples only when necessary to satisfy the tolerance constraint. This reduces the auditor's optimization to determining how its budget should be allocated between the two sensitive groups.

Solving the resulting constrained optimization yields closed-form expressions for the monitored samples $(n_u,n_p)$. Substituting these into the company's objective together with its optimal manipulation strategy gives the stated lower bound.

\subsection{Practical Variant: Balanced Audit Sets}

Although Theorem~\ref{thm:theorem_2} applies to general audit sets, many practical auditing protocols construct balanced audit sets. Under this common assumption we obtain a simpler closed-form bound.

\begin{assumption}
The auditor ensures that the chosen set $S_{B,\alpha}$ contains an equal number of elements from both the sensitive attribute groups.     
\label{as:as1}
\end{assumption}

\begin{proposition}[\textbf{Lower Bound under Assumption \ref{as:as1} and Budgeted $\alpha$-Tolerant Audit}]
\label{prop:prop_1}
Let $h^\star$ be the initial classifier and the auditor have a budget \( B \) under $\alpha$-tolerance. Define \( t := \Delta(h^\star) + \alpha \). Then, under Assumption~\ref{as:as1}, the maximum increase in fairness violation achievable by the company is lower bounded by:
\[
\Delta(h') - \Delta(h^\star) \geq \frac{N - \frac{B}{2}(1 + t)}{s_0},
\]
subject to appropriate feasibility constraints on \( B \). The bound is tight when $s_0 = s_1$. $h'$ is the classifier that the company ends with. 
\end{proposition}

\paragraph{Interpretation.} Balanced audit sets reduce the flexibility available to the auditor, but they do not eliminate strategic manipulation.

The resulting lower bound continues to scale with the number of unaudited samples, showing that balanced representation alone cannot guarantee robust fairness certification under finite audit budgets. The bound remains tight when the two sensitive groups are balanced.

\section{Illustrating the Theoretical Limits}
To illustrate the theoretical limits derived in the previous section, we instantiate the audit constraints using a simple audit-set construction heuristic~(\textsf{RASC}). The heuristic ranks samples according to their susceptibility to post-audit manipulation and constructs an audit set by prioritizing those that are most likely to influence the classifier's decision boundary. The audit set is then iteratively refined until the required $\alpha$-tolerance constraint is satisfied. Algorithm~\ref{alg:rasc} summarizes the procedure.

For the experiments in this section, the ranking criterion is instantiated using the geometry of the underlying classifier. For linear models, samples are ranked by their distance to the decision boundary, whereas for neural networks, they are ranked using the classifier's prediction confidence. The heuristic serves only to instantiate the audit constraints used in our theoretical analysis and is not intended as a novel auditing algorithm.

\begin{algorithm}[t]
\DontPrintSemicolon
\small
\caption{Replacement-based Auditing Set Construction (\textsf{RASC})}
\label{alg:rasc}

\KwIn{Dataset $D$, classifier $h^\star$, audit budget $B$, tolerance $\alpha$}
\KwOut{Audit set $S$}

Partition $D$ into $\{X_{00},X_{01},X_{10},X_{11}\}$ according to the predictions of $h^\star$\;

Rank samples in each partition according to their proximity to the decision boundary\;

Initialize an audit set $S$ of size at most $B$ using the highest-priority samples\;

\While{$|\Delta(S,h^\star)-\Delta(h^\star)|>\alpha$}{
    Replace samples in $S$ with the next highest-priority candidates to reduce the fairness estimation error\;
}

\Return{$S$}

\end{algorithm}

\label{sec:expt}

\pgfplotscreateplotcyclelist{mycolorlist}{%
    {blue, mark=*},
    {red, mark=square*},
    {green!50!black, mark=triangle*},
    {purple!70!blue, mark=diamond*},
    {orange, mark=otimes*},
    {brown!60!black, mark=oplus*},
    {teal, mark=pentagon*},
    {magenta!60!black, mark=x}
}

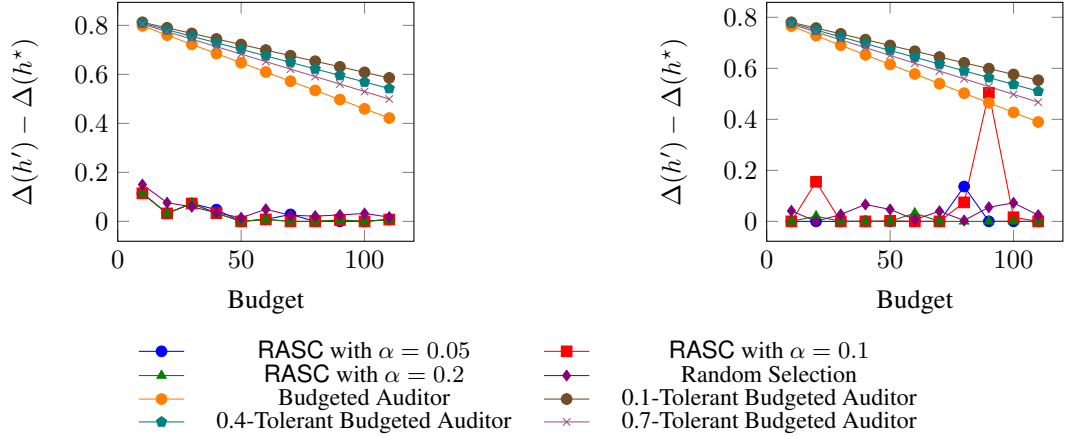
\begin{figure*}[!t]
    \centering

    \begin{minipage}[b]{0.48\textwidth}
        \centering
        \begin{tikzpicture}
            \begin{axis}[
                xlabel={Budget},
                ylabel={$\Delta(h') - \Delta(h^\star)$},
                grid=minor,
                width = 5.5cm,
                cycle list name=mycolorlist,
            ]
            \addplot table[x index=0, y index=1, col sep = space]{Linear_Model/varying_budget_fair.txt};
            \addplot table[x index=0, y index=2, col sep = space]{Linear_Model/varying_budget_fair.txt};
            \addplot table[x index=0, y index=3, col sep = space]{Linear_Model/varying_budget_fair.txt};
            \addplot table[x index=0, y index=4, col sep = space]{Linear_Model/varying_budget_fair.txt};
            \addplot table[x index=0, y index=5, col sep = space]{Linear_Model/varying_budget_fair.txt};
            \addplot table[x index=0, y index=6, col sep = space]{Linear_Model/varying_budget_fair.txt};
            \addplot table[x index=0, y index=7, col sep = space]{Linear_Model/varying_budget_fair.txt};
            \addplot table[x index=0, y index=8, col sep = space]{Linear_Model/varying_budget_fair.txt};
            \end{axis}
        \end{tikzpicture}
    \end{minipage}
    \hfill
    \begin{minipage}[b]{0.48\textwidth}
        \centering
        \begin{tikzpicture}
            \begin{axis}[
                xlabel={Budget},
                ylabel={$\Delta(h') - \Delta(h^\star)$},
                grid=minor,
                width = 5.5cm,
                cycle list name=mycolorlist,
            ]
            \addplot table[x index=0, y index=1, col sep = space]{NN/varying_budget_fair.txt};
            \addplot table[x index=0, y index=2, col sep = space]{NN/varying_budget_fair.txt};
            \addplot table[x index=0, y index=3, col sep = space]{NN/varying_budget_fair.txt};
            \addplot table[x index=0, y index=4, col sep = space]{NN/varying_budget_fair.txt};
            \addplot table[x index=0, y index=5, col sep = space]{NN/varying_budget_fair.txt};
            \addplot table[x index=0, y index=6, col sep = space]{NN/varying_budget_fair.txt};
            \addplot table[x index=0, y index=7, col sep = space]{NN/varying_budget_fair.txt};
            \addplot table[x index=0, y index=8, col sep = space]{NN/varying_budget_fair.txt};
            \end{axis}
        \end{tikzpicture}
    \end{minipage}

    \begin{tikzpicture}
        \begin{axis}[
            hide axis,
            width=5cm,
            height=4cm,
            legend columns=2,
            legend style={
                draw=none,
                font=\footnotesize,
                row sep=-3pt,
                column sep=10pt
            },
            cycle list name=mycolorlist
        ]
        \addplot coordinates {(0,0)}; \addlegendentry{\textsf{RASC} with $\alpha = 0.05$}
        \addplot coordinates {(0,0)}; \addlegendentry{\textsf{RASC} with $\alpha = 0.1$}
        \addplot coordinates {(0,0)}; \addlegendentry{\textsf{RASC} with $\alpha = 0.2$}
        \addplot coordinates {(0,0)}; \addlegendentry{Random Selection}
        \addplot coordinates {(0,0)}; \addlegendentry{Budgeted Auditor}
        \addplot coordinates {(0,0)}; \addlegendentry{0.1-Tolerant Budgeted Auditor}
        \addplot coordinates {(0,0)}; \addlegendentry{0.4-Tolerant Budgeted Auditor}
        \addplot coordinates {(0,0)}; \addlegendentry{0.7-Tolerant Budgeted Auditor}
        \end{axis}
    \end{tikzpicture}

        \caption{(Left) Post-audit manipulation by the Logistic Regression class with varying budget (Right) Post-audit manipulation by the Neural Network class with varying budget}
    \label{fig:combined_budget_plots}
\end{figure*}

\pgfplotscreateplotcyclelist{mycolorlist}{
    {blue, mark=*},
    {red, mark=square*},
    {green!60!black, mark=triangle*},
    {orange, mark=diamond*},
}

\begin{figure*}[!t]
    \centering
    \begin{minipage}[b]{0.48\textwidth}
        \centering
        \begin{tikzpicture}
            \begin{axis}[
                xlabel={$\alpha$},
                ylabel={$\Delta(h') - \Delta(h^\star)$},
                grid=minor,
                width = 5.5cm,
                cycle list name=mycolorlist,
            ]
            \addplot table[x index=0, y index=1]{Linear_Model/varying_alpha_biased.txt};
            \addplot table[x index=0, y index=2]{Linear_Model/varying_alpha_biased.txt};
            \addplot table[x index=0, y index=3]{Linear_Model/varying_alpha_biased.txt};
            \addplot table[x index=0, y index=4]{Linear_Model/varying_alpha_biased.txt};
            \end{axis}
        \end{tikzpicture}
    \end{minipage}
    \hfill
    \begin{minipage}[b]{0.48\textwidth}
        \centering
        \begin{tikzpicture}
            \begin{axis}[
                xlabel={$\alpha$},
                ylabel={$\Delta(h') - \Delta(h^\star)$},
                grid=minor,
                width = 5.5cm,
                cycle list name=mycolorlist,
            ]
            \addplot table[x index=0, y index=1]{NN/varying_alpha_fair.txt}; 
            \addplot table[x index=0, y index=2]{NN/varying_alpha_fair.txt}; 
            \addplot table[x index=0, y index=3]{NN/varying_alpha_fair.txt}; 
            \addplot table[x index=0, y index=4]{NN/varying_alpha_fair.txt};
            \end{axis}
        \end{tikzpicture}
    \end{minipage}

    \begin{tikzpicture}
        \begin{axis}[
            hide axis,
            width=12cm,
            height=2cm,
            legend columns=3,
            legend style={
                draw=none,
                font=\small,
                row sep=-3pt,
                column sep=10pt
            },
            cycle list name=mycolorlist
        ]
        \addplot coordinates {(0,0)}; \addlegendentry{\textsf{RASC}}
        \addplot coordinates {(0,0)}; \addlegendentry{Random Selection}
        \addplot coordinates {(0,0)}; \addlegendentry{Budgeted Auditor}
        \addplot coordinates {(0,0)}; \addlegendentry{$\alpha$-Tolerant Budgeted Auditor}
        \end{axis}
    \end{tikzpicture}

    \caption{(Left) Post-audit manipulation by the Logistic Regression class with varying values of $\alpha$ (Right) Post-audit manipulation by the Neural Network class with varying values of $\alpha$}
    \label{fig:combined_alpha_plots}
\end{figure*}
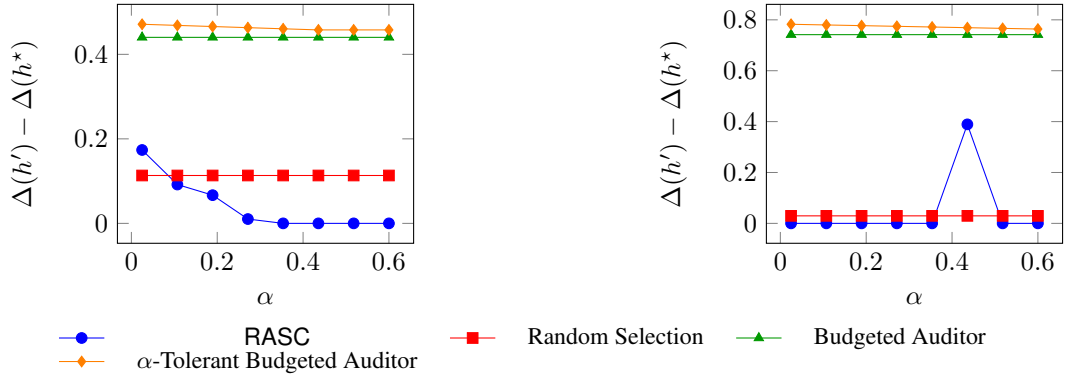

\begin{figure*}[t]
    \centering

    \begin{minipage}[b]{0.48\textwidth}
        \centering
        \begin{tikzpicture}
            \begin{axis}[
                xlabel={$\Delta(h^\star)$},
                ylabel={$\Delta(h')$},
                grid=minor,
                width = 5.5cm,
                cycle list name=mycolorlist,
            ]
            \addplot table[x index=0, y index=1]{Linear_Model/varying_dp.txt}; 
            \addplot table[x index=0, y index=2]{Linear_Model/varying_dp.txt}; 
            \addplot table[x index=0, y index=3]{Linear_Model/varying_dp.txt}; 
            \addplot table[x index=0, y index=4]{Linear_Model/varying_dp.txt}; 
            \end{axis}
        \end{tikzpicture}
    \end{minipage}
    \hfill
    \begin{minipage}[b]{0.48\textwidth}
        \centering
        \begin{tikzpicture}
            \begin{axis}[
                xlabel={$\Delta(h^\star)$},
                ylabel={$\Delta(h')$},
                grid=minor,
                width = 5.5cm,
                cycle list name=mycolorlist,
            ]
            \addplot table[x index=0, y index=1]{NN/varying_dp.txt}; 
            \addplot table[x index=0, y index=2]{NN/varying_dp.txt}; 
            \addplot table[x index=0, y index=3]{NN/varying_dp.txt}; 
            \addplot table[x index=0, y index=4]{NN/varying_dp.txt}; 
            \end{axis}
        \end{tikzpicture}
    \end{minipage}
    \hfill

    \begin{tikzpicture}
        \begin{axis}[
            hide axis,
            width=12cm,
            height=2cm,
            legend columns=3,
            legend style={
                draw=none,
                font=\small,
                row sep=-3pt,
                column sep=10pt
            },
            cycle list name=mycolorlist
        ]
        \addplot coordinates {(0,0)}; \addlegendentry{\textsf{RASC}}
        \addplot coordinates {(0,0)}; \addlegendentry{Random Selection}
        \addplot coordinates {(0,0)}; \addlegendentry{Budgeted Auditor}
        \addplot coordinates {(0,0)}; \addlegendentry{$0.05$-Tolerant Budgeted Auditor}
        \end{axis}
    \end{tikzpicture}

    \caption{(Left) Post-audit manipulation values for the Logistic Regression class with varying values of $\Delta(h^\star)$ (Right) Post-audit manipulation values for the Neural Network class in the learned space with varying values of $\Delta(h^\star)$}
    \label{fig:combined_plots}
\end{figure*}
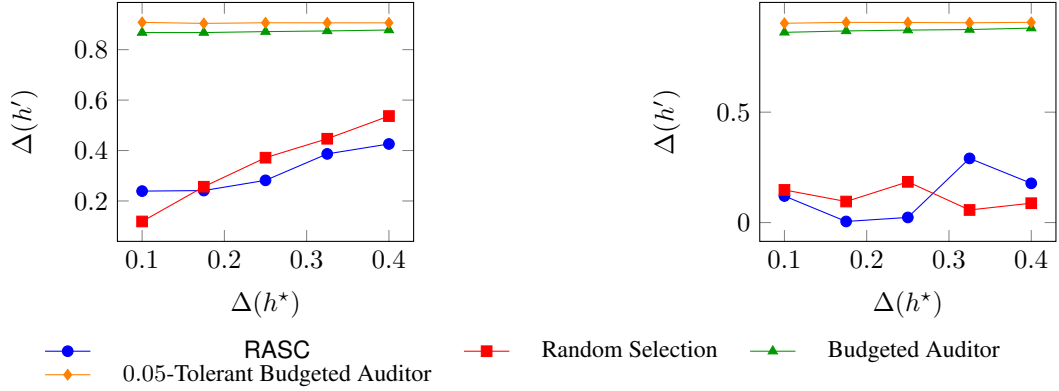

\subsection{Datasets and Baselines}

\noindent\textbf{Datasets.}
We evaluate the proposed audit-set construction heuristic on three benchmark datasets: \textbf{Student Performance}~\cite{stu_por}, \textbf{German Credit Risk}~\cite{german}, and \textbf{COMPAS Recidivism}~\cite{compas}. We use \emph{gender} as the sensitive attribute for Student Performance and German Credit Risk, and \emph{race} for COMPAS.

\noindent\textbf{Experimental Setup.}
We compare the proposed heuristic against random audit-set selection under varying audit budgets \(B\), tolerance levels \(\alpha\), and initial demographic parity violations \(\Delta(h^\star)\). For each audit set, we compute the most biased classifier \(h'\) that remains consistent with the audit set using the reduction-based approach of \cite{reductions}, following \cite{yan2022active}. Random auditing is repeated 10 times and the average post-audit deviation is reported.

For reference, we also plot the theoretical lower bounds from Theorem~\ref{thm:theorem_1} and Proposition~\ref{prop:prop_1}, which correspond to a computationally unbounded company operating on the same group-wise counts \(x_{ij}\). Unless otherwise stated, results are shown for the Student Performance dataset; analogous trends for the remaining datasets are provided in the supplementary.

\pgfplotscreateplotcyclelist{mycolorlist}{
    {blue, mark=*},
    {red, mark=square*},
    {green!60!black, mark=triangle*},
    {orange, mark=diamond*},
}

\subsection{Results}

\noindent\textbf{Effect of Audit Budget.}
Figure~\ref{fig:combined_budget_plots} shows the post-audit demographic parity (DP) violation as the audit budget \(B\) increases. Across both logistic regression and neural network classifiers, larger audit budgets reduce the company's ability to manipulate the deployed model, consistent with the theoretical lower bounds. The proposed audit-set construction heuristic consistently achieves lower post-audit deviation than random auditing, with the improvement being more pronounced for the more expressive neural network models.

\noindent\textit{Comparison with Theory.}
The empirical deviations are consistently below the theoretical bounds derived in Theorem~\ref{thm:theorem_1} and Proposition~\ref{prop:prop_1}. This is expected since the theoretical analysis assumes a computationally unbounded company, whereas the experiments restrict manipulation to finite hypothesis classes. Nevertheless, the empirical trends closely mirror the theoretical prediction that increasing the audit budget reduces, but does not eliminate, post-audit manipulation.

\vspace{1mm}

\noindent\textbf{Effect of Fairness Tolerance.}
Figure~\ref{fig:combined_alpha_plots} examines the influence of the tolerance parameter \(\alpha\) with a fixed audit budget. As \(\alpha\) increases, the auditor has greater flexibility in constructing representative audit sets, leading to lower post-audit manipulation for logistic regression models. For neural networks, the same qualitative trend is less pronounced, reflecting the increased expressivity of the underlying hypothesis class.

\noindent\textit{Comparison with Theory.}
The theoretical analysis predicts that larger tolerance allows the $\alpha$-tolerant auditor to approach the behavior of the budget-only auditor while producing audit sets that remain representative of the certified classifier. The empirical results broadly support this behavior.

\vspace{1mm}

\noindent\textbf{Effect of Initial Fairness.}
Figure~\ref{fig:combined_plots} studies the effect of the initial demographic parity violation \(\Delta(h^\star)\). Under finite hypothesis classes, the post-audit DP generally increases with the initial unfairness of the certified classifier, indicating that starting from a more biased model can benefit a strategic company. This effect is substantially reduced under tighter $\alpha$-tolerant auditing.

\noindent\textit{Comparison with Theory.}
In contrast, the theoretical bounds are largely insensitive to the initial fairness of \(h^\star\), since an unrestricted company can already realize near-optimal manipulation even from relatively fair starting points. The discrepancy highlights the gap between worst-case guarantees and practical manipulation under finite hypothesis classes.

\section{Discussion and Conclusion}

The lower bounds derived in this paper provide a quantitative characterization of the limitations of finite-budget fairness auditing. A key insight from Theorem~\ref{thm:theorem_1} is that the difficulty of auditing is fundamentally determined by the number of samples that remain outside the audit set. Every unaudited sample provides an opportunity for strategic post-audit manipulation. This interpretation also explains why increasing the audit budget produces diminishing returns; unless the auditor is able to certify nearly all manipulable samples, some degree of manipulation is unavoidable.

Theorem~\ref{thm:theorem_2} further shows that requiring audit sets to accurately estimate the fairness of the certified classifier introduces an inherent trade-off. In the budget-only setting, the auditor is free to concentrate its budget entirely on the samples most valuable to the company. The $\alpha$-tolerance constraint restricts this flexibility by requiring the audit set to remain representative of the certified model. Consequently, the auditor sacrifices some ability to directly block manipulation in exchange for producing a more faithful fairness certificate. At the same time, representative certificates expose the fairness of the initially certified classifier, discouraging companies from beginning with highly biased models. Thus, $\alpha$-tolerant auditing balances two competing objectives: limiting post-audit manipulation and providing an informative fairness certificate.

Our empirical study suggests that these theoretical insights persist even under practical hypothesis classes. Although the observed post-audit deviations are substantially smaller than the worst-case lower bounds, they exhibit the same qualitative dependence on the audit budget, fairness tolerance, and initial classifier fairness. This gap is expected, since the theoretical analysis assumes a computationally unbounded company, whereas practical models are constrained by finite hypothesis classes.

Finally, our analysis focuses on static, one-shot fairness certification. An important direction for future work is to understand how richer auditing protocols, such as repeated audits, deployment-time monitoring, or multiple independent information sources~\cite{bourrée2025mitigatingfairwashingusingtwosource,iterative}, can reduce the gap between practical auditing procedures and the worst-case limits characterized in this work.

\bibliographystyle{plain} 
\bibliography{References}

\begin{thebibliography}{10}

\bibitem{reductions}
A.~Agarwal, A.~Beygelzimer, M.~Dudik, J.~Langford, and H.~Wallach.
\newblock A reductions approach to fair classification.
\newblock In Jennifer Dy and Andreas Krause, editors, {\em Proceedings of the
  35th ICML}, volume~80, pages 60--69, 2018.

\bibitem{angwin2016machine}
J.~Angwin, J.~Larson, S.~Mattu, and L.~Kirchner.
\newblock Machine bias: There’s software used across the country to predict
  future criminals. and it’s biased against blacks.
\newblock {\em ProPublica}, 2016.

\bibitem{compas}
J.~Angwin, J.~Larson, S.~Mattu, and L.~Kirchner.
\newblock Propublica compas recidivism data.
\newblock \url{https://github.com/propublica/compas-analysis}, 2016.

\bibitem{barocas16}
S.~Barocas and A.~D. Selbst.
\newblock Big data's disparate impact.
\newblock {\em Cal. L. Rev.}, 104:671, 2016.

\bibitem{bourree2025robust}
Bourr{\'e, J. G.}e, Augustin Godinot, Martijn~De Vos, Milos Vujasinovic, Sayan
  Biswas, Gilles Tredan, Erwan~Le Merrer, and Anne-Marie Kermarrec.
\newblock Robust {ML} auditing using prior knowledge.
\newblock In {\em ICML Workshop on Technical AI Governance (TAIG)}, 2025.

\bibitem{bourrée2025mitigatingfairwashingusingtwosource}
J.~G. Bourrée, E.~L. Merrer, G.~Tredan, and B.~Rottembourg.
\newblock Mitigating fairwashing using two-source audits, 2025.

\bibitem{Casper_2024}
S.~Casper, C.~Ezell, C.~Siegmann, N.~Kolt, T.~L. Curtis, B.~Bucknall, A.~Haupt,
  K.~Wei, J.~Scheurer, M.~Hobbhahn, L.~Sharkey, S.~Krishna, M.~Von~Hagen,
  S.~Alberti, A.~Chan, Q.~Sun, M.~Gerovitch, D.~Bau, M.~Tegmark, D.~Krueger,
  and D.~Hadfield-Menell.
\newblock Black-box access is insufficient for rigorous ai audits.
\newblock In {\em The 2024 FAccT}, FAccT ’24, page 2254–2272, 2024.

\bibitem{betting}
B.~Chugg, S.~Cortes-Gomez, B.~Wilder, and A.~Ramdas.
\newblock Auditing fairness by betting.
\newblock In {\em Proceedings of the 37th International Conference on Neural
  Information Processing Systems}, 2023.

\bibitem{stu_por}
P.~Cortez.
\newblock {Student Performance}.
\newblock UCI Machine Learning Repository, 2008.

\bibitem{nyc_local_law_144}
{New Y. C.} Council.
\newblock New york city local law 144 of 2021: Automated employment decision
  tools.
\newblock Local Law No. 144 of 2021, effective January 1, 2023, 2021.

\bibitem{dwork12}
C.~Dwork, M.~Hardt, T.~Pitassi, O.~Reingold, and R.~Zemel.
\newblock Fairness through awareness.
\newblock In {\em Proceedings of the 3rd Innovations in Theoretical Computer
  Science Conference}, pages 214--226, 2012.

\bibitem{fabris2025fairness}
A.~Fabris, N.~Baranowska, M.~J. Dennis, D.~Graus, P.~Hacker, J.~Saldivar,
  F.~Zuiderveen~Borgesius, and A.~J. Biega.
\newblock Fairness and bias in algorithmic hiring: A multidisciplinary survey.
\newblock {\em ACM Transactions on Intelligent Systems and Technology},
  16(1):1--54, 2025.

\bibitem{fukuchi2020faking}
K.~Fukuchi, S.~Hara, and T.~Maehara.
\newblock Faking fairness via stealthily biased sampling.
\newblock In {\em AAAI}, volume~34, pages 412--419, 2020.

\bibitem{auditing_audits}
M.~K. Gerchick, Ro~Encarnaci\'{o}n, Cole Tanigawa-Lau, Lena Armstrong, Ana
  Guti\'{e}rrez, and Dana\'{e} Metaxa.
\newblock Auditing the audits: Lessons for algorithmic accountability from
  local law 144's bias audits.
\newblock In {\em Proceedings of the 2025 FAccT}, page 29–44, 2025.

\bibitem{Godinot_2024}
A.~Godinot, E.~L. Merrer, G.~Trédan, C.~Penzo, and F.~Taïani.
\newblock Under manipulations, are some ai models harder to audit?
\newblock In {\em 2024 SaTML (SaTML)}, page 644–664, 2024.

\bibitem{german}
H.~Hofmann.
\newblock {Statlog (German Credit Data)}.
\newblock UCI Machine Learning Repository, 1994.

\bibitem{iterative}
P.~Maneriker, C.~Burley, and S.~Parthasarathy.
\newblock Online fairness auditing through iterative refinement.
\newblock In {\em Proceedings of the 29th KDD}, page 1665–1676, 2023.

\bibitem{eu_ai_act_2024}
{European and Union, C. o. t. E.} Parliament.
\newblock Regulation (eu) 2024/1689: Artificial intelligence act.
\newblock Official Journal of the European Union, 2024.

\bibitem{shahin2022washing}
A.~Shahin~Shamsabadi, M.~Yaghini, N.~Dullerud, S.~Wyllie, Ulrich A{\"\i}vodji,
  Aisha Alaagib, S{\'e}bastien Gambs, and Nicolas Papernot.
\newblock Washing the unwashable: On the (im) possibility of fairwashing
  detection.
\newblock {\em NeurIPS}, 35:14170--14182, 2022.

\bibitem{pragmatic_fairness}
J.~J. Smith, M.~Madaio, R.~Burke, and C.~Fiesler.
\newblock Pragmatic fairness: Evaluating ml fairness within the constraints of
  industry.
\newblock In {\em Proceedings of the 2025 FAccT}, page 628–638, 2025.

\bibitem{yan2022active}
T.~Yan and C.~Zhang.
\newblock Active fairness auditing.
\newblock In {\em ICML}, pages 24929--24962, 2022.

\bibitem{zhang2021understanding}
C.~Zhang, S.~Bengio, M.~Hardt, B.~Recht, and O.~Vinyals.
\newblock Understanding deep learning (still) requires rethinking
  generalization.
\newblock {\em Communications of the ACM}, 64(3):107--115, 2021.

\end{thebibliography}

\appendix
\section{Omitted Proofs}
\label{app:proofs}
We present the proofs of Theorem~\ref{thm:theorem_1}, Theorem~\ref{thm:theorem_2}, and Proposition~\ref{prop:prop_1}. Throughout the appendix, let \(U\) denote the set of manipulated samples from sensitive group \(0\), and \(P\) denote the manipulated samples from sensitive group \(1\). Let \(u=|U|\) and \(p=|P|\) denote the corresponding numbers of manipulated samples, while \(n_u\) and \(n_p\) denote the numbers of monitored samples selected by the auditor from \(U\) and \(P\), respectively.

\subsection{Proof of Theorem~\ref{thm:theorem_1}}

\begin{proof} Assume, without loss of generality, that
\[
\frac{x_{01}}{s_0}>\frac{x_{11}}{s_1},
\]
so increasing predictions in group \(0\) and decreasing predictions in group \(1\) increases the demographic parity violation. 

After the company manipulates \(u\) samples from \(U\) and \(p\) samples from \(P\), and the auditor monitors \(n_u\) and \(n_p\) of these samples, respectively, the post-audit increase in demographic parity becomes
\[
\Delta(h')-\Delta(h^\star)
=
\frac{u-n_u}{s_0}
+
\frac{p-n_p}{s_1},
\] where \(n_u+n_p\le B\). The auditor therefore allocates its budget to maximize \(
\frac{n_u}{s_0}+\frac{n_p}{s_1},
\) giving priority to the sensitive group with the smaller denominator.

\medskip

\noindent\textbf{Case 1: \(s_0<s_1\).}
The auditor first monitors samples from \(U\). This yields

\[
\Delta(h')-\Delta(h^\star)\ge
\begin{cases}
\displaystyle
\frac{u-B}{s_0}+\frac{p}{s_1},
&
B\le u,
\\[2mm]
\displaystyle
\frac{u+p-B}{s_1},
&
u<B\le u+p,
\\[2mm]
0,
&
B>u+p.
\end{cases}
\]

\medskip

\noindent\textbf{Case 2: \(s_1<s_0\).}
By symmetry, the auditor first monitors samples from \(P\), giving

\[
\Delta(h')-\Delta(h^\star)\ge
\begin{cases}
\displaystyle
\frac{p-B}{s_1}+\frac{u}{s_0},
&
B\le p,
\\[2mm]
\displaystyle
\frac{u+p-B}{s_0},
&
p<B\le u+p,
\\[2mm]
0,
&
B>u+p.
\end{cases}
\]

\medskip

\noindent\textbf{Case 3: \(s_0=s_1\).}
Both groups contribute equally, and therefore

\[
\Delta(h')-\Delta(h^\star)\ge
\begin{cases}
\displaystyle
\frac{u+p-B}{s_0},
&
B\le u+p,
\\[2mm]
0,
&
B>u+p.
\end{cases}
\]

Finally, the company's optimal manipulation is obtained by choosing
\[
u=x_{00},
\qquad
p=x_{11},
\]
so that
\[
u+p=x_{00}+x_{11}=N.
\]
Combining the above cases and using
\[
\frac1{s_{\max}}
\le
\frac1{s_0},
\qquad
\frac1{s_{\max}}
\le
\frac1{s_1},
\]
gives

\[
\Delta(h')-\Delta(h^\star)
\ge
\max\left\{
\frac{N-B}{s_{\max}},
0
\right\}.
\]

When \(s_0=s_1\), the inequality becomes an equality, establishing tightness.
\end{proof}

\subsection{Proof of Theorem~\ref{thm:theorem_2}}
The proof proceeds in three steps. We first characterize the structure of an optimal audit set under the $\alpha$-tolerance constraint. We then solve the auditor's optimization problem for a fixed manipulation proposed by the company, and finally substitute the company's optimal manipulation strategy to obtain the stated lower bound.

\begin{lemma}[Structure of an Optimal Audit Set]
For an optimal $\alpha$-tolerant auditor, every sample outside the manipulation sets $U$ and $P$ is selected from $X_{01}$ or $X_{10}$. Furthermore, the resulting audit set makes sensitive group $1$ the privileged group.
\end{lemma}

\begin{proof} The auditor maximizes the number of monitored manipulated samples while satisfying the $\alpha$-tolerance constraint.

Replacing samples from $X_{00}$ or $X_{11}$ by manipulated samples from $U$ or $P$ can only increase the number of monitored deviations without affecting feasibility. Hence an optimal audit set never contains samples from $X_{00}$ or $X_{11}$ outside the manipulation sets.

Starting from an audit set consisting entirely of manipulated samples (with a fraction $f_1$ samples from $P$ and a fraction $f_0$ samples from $U$), the auditor introduces samples from $X_{01}$ and $X_{10}$ only as needed to satisfy the tolerance constraint. Since this requires the smallest number of replacements, it is optimal. The resulting audit set necessarily satisfies

\[
\frac{f_1p-n_1}{f_1p}
-
\frac{n_0}{f_0u}
=t,
\]

where $f_1p + f_0u = B$.
\end{proof}

\begin{lemma}[Optimal Replacement Strategy]
For initial fixed fractions $f_1$ and $f_0$ for samples from $P$ and $U$ respectively , the optimal replacement satisfies

\[
n_1=x_{10},
\qquad
n_0
=
f_0u
\left(
1-t-\frac{x_{10}}{f_1p}
\right).
\]
\end{lemma}

\begin{proof} Using the tolerance constraint,

\[
n_0
=
f_0u
\left(
1-\frac{n_1}{f_1p}-t
\right),
\] the auditor's objective becomes a linear function of $n_1$.

The feasible interval for $n_1$ is determined by
\(0\le n_0\le x_{01}\) and \(0\le n_1\le x_{10}\). Under the feasibility assumptions of the theorem, the objective is monotone over this interval and therefore attains its optimum at the boundary, \(
n_1=x_{10}.
\) Substituting this value yields the stated expression for \(n_0\).
\end{proof}

\begin{lemma}[Optimal Budget Allocation] The optimal allocation of the audit budget satisfies
\[
f_1
=
\frac{x_{10}}
{(1-t)p},
\qquad
f_0
=
\frac1{u}
\left(
B-\frac{x_{10}}{1-t}
\right),
\] provided the feasibility conditions of the theorem hold.

\end{lemma}

\begin{proof} Substituting the optimal values of \(n_0\) and \(n_1\) reduces the auditor's objective to

\[
g(f_1)
=
f_1p
\left(
\frac1{s_1}
-
\frac{t}{s_0}
\right)
+
\frac{Bx_{10}}
{s_0f_1p}.
\]

The feasible interval for \(f_1\) is determined by the budget constraint and the non-negativity of \(n_0\). Since \(g(f_1)\) has no interior maximum over this interval, its optimum is attained at an endpoint. Under the theorem's feasibility conditions, the lower endpoint is optimal, yielding \(
f_1
=
\frac{x_{10}}
{(1-t)p},
\) and the expression for \(f_0\) follows immediately from \(
f_0u
+
f_1p
=
B.
\)
\end{proof}

\begin{proof} The company's objective is

\[
\frac{u-n_u}{s_0}
+
\frac{p-n_p}{s_1}.
\]  Clearly, \( u=x_{00}\) and \( p=x_{11} \) maximize this quantity. Substituting the optimal values of \(n_u\) and \(n_p\) obtained from Lemma~A.3 yields
\[
\Delta(h')-\Delta(h^\star)
\ge
\frac{N-B}{s_{\max}}
+
x_{10}
\frac{s_1-s_0t}
{s_0s_1(1-t)}.
\]

Since the bound decreases monotonically with \(t\), the auditor chooses the largest feasible value, \( t = \Delta(h^\star)+\alpha, \) which completes the proof.
\end{proof}

\subsection{Proof of Proposition~\ref{prop:prop_1}}

\begin{proof} Under Assumption~\ref{as:as1}, the auditor allocates exactly half of its budget to each sensitive group. Consequently, \[
n_u+n_0=\frac{B}{2},
\qquad
n_p+n_1=\frac{B}{2},
\] where \(n_u,n_p\) denote monitored manipulated samples and \(n_0,n_1\) denote monitored non-manipulated samples. The $\alpha$-tolerance constraint requires \(
\frac{n_p}{n_p+n_1}
-
\frac{n_0}{n_u+n_0}
=
t,
\) where \(\Delta(h^\star)-\alpha\ \leq t \leq \Delta(h^\star)+\alpha\). Combining the above equalities immediately gives \(
n_p-n_0=\frac{Bt}{2},
\) and therefore \(
n_u+n_p
=
\frac{B(1+t)}{2}.
\)

Hence, for a fixed manipulation proposed by the company, the auditor's optimization reduces to maximizing \[
\frac{n_u}{s_0}
+
\frac{n_p}{s_1},
\] subject to the above constraints. When \(s_0<s_1\), every monitored sample from group \(0\) blocks a larger increase in demographic parity than one from group \(1\). The auditor therefore allocates as many manipulated samples as possible to \(n_u\). By symmetry, when \(s_1<s_0\), the auditor instead prioritizes \(n_p\). Finally, when \(s_0=s_1\), every feasible allocation yields the same objective value.

In all three cases, the resulting post-audit deviation satisfies
\[
\Delta(h')-\Delta(h^\star)
\ge
\max\left\{
\frac{u+p-\frac{B(1+t)}{2}}
{s_{\max}},
\,0
\right\},
\] subject to the feasibility conditions stated in the proposition.

Finally, the company maximizes the deviation by choosing \(
u=x_{00}\),
\(
p=x_{11},
\) so that \(
u+p
=
x_{00}+x_{11}
=
N.
\) Substituting these values gives

\[
\Delta(h')-\Delta(h^\star)
\ge
\max\left\{
\frac{N-\frac{B}{2}(1+t)}
{s_{\max}},
\,0
\right\}.
\] Since \(
\frac{1}{s_{\max}}
\le
\frac{1}{s_0},
\) we obtain the stated lower bound

\[
\Delta(h')-\Delta(h^\star)
\ge
\frac{N-\frac{B}{2}(1+t)}
{s_0}.
\]

Finally, the bound is decreasing in \(t\), and therefore the auditor chooses the largest admissible value,

\[
t=\Delta(h^\star)+\alpha.
\]

When \(s_0=s_1\), the inequality becomes an equality, establishing the tightness of the bound.

\end{proof}

\section{Architecture of the Neural Network used in the Experiments}
The neural network consists of three fully connected layers with \texttt{tanh} activations. It is defined as
\[
f(x)=\mathbf{W}_3^\top
\left(
\tanh\!\left(
\mathbf{W}_2^\top
\left(
\tanh(\mathbf{W}_1^\top x+\mathbf{b}_1)
\right)
+\mathbf{b}_2
\right)
\right)
+\mathbf{b}_3,
\]
where
\[
\mathbf{W}_1\in\mathbb{R}^{i\times50},\qquad
\mathbf{W}_2\in\mathbb{R}^{50\times30},\qquad
\mathbf{W}_3\in\mathbb{R}^{30\times2},
\]
and \(\mathbf{b}_1,\mathbf{b}_2,\mathbf{b}_3\) are the corresponding bias vectors. Here, \(i\) denotes the input dimension of the dataset.
\section{Additional Experiments}
\label{sec:add_exp}

We present results for the COMPAS Recidivism Racial Bias and the German Credit Risk datasets in this section. We experiment with only the Logistic Regression class for these datasets to better illustrate the trends in deviation achieved by the company with varying DP violation  of the initial classifier $\Delta(h^\star)$. We explain observations for only the COMPAS Recidivism Racial Bias dataset, since the trends are very similar for the German Credit Risk dataset.

\subsection{COMPAS Recidivism Racial Bias Dataset}

\textbf{Effect of Budget \( B \).} We evaluate the deviation by varying the audit budget \( B \) while keeping the tolerance level \( \alpha \) fixed. As shown in Figure~\ref{fig:combined_budget_plots_compas} (Left), the deviation is very low for budgets greater than about 50. There is a slight benefit for the auditor to using \textsf{RASC} vs random selection, due to the heuristic working more naturally for linear classifiers. Also, there is little change in the deviation with varying value of $\alpha$, as all the approaches (including random selection) restrict the company's deviation strongly. As a result, the gains for the auditor obtained by increasing the value of $\alpha$ are less pronounced. 

\noindent\textit{Comparison with Theoretical Bounds.} Again, the empirical deviations observed are substantially more conservative than the theoretical bounds derived in Theorem 1 and Proposition \ref{prop:prop_1}. This is expected due to the limited expressivity of the empirical hypothesis class. Moreover, the Budgeted and $\alpha$-Tolerant auditors give very similar bounds, due to the significantly lower values of the budget $B$ with respect to the overall dataset. 

\noindent\textit{Effect of Initial Classifier Fairness.} A similar trend is observed when the initial classifier is more biased, though the gap between RASC and random selection is higher due to the ease of auditing a more biased linear classifier. Assuming a distribution with non-trivial variance along all dimensions (like most datasets), a company with limited expressivity in its hypothesis class will have progressively more difficulty in deviating as $\Delta(h^\star)$ increases.

\textbf{Effect of Tolerance $\alpha$.} We examine the impact of varying tolerance \( \alpha \) on \textsf{RASC} with fixed budget ($25$ samples). As shown in Figure~\ref{fig:combined_alpha_plots_compas} (Left), the company's ability to deviate worsens as \( \alpha \) increases, plateauing beyond a threshold. 

\noindent\textit{Effect of Initial Classifier Bias.} As visible in Figure \ref{fig:combined_alpha_plots_compas} (Right), \( \alpha \)'s influence is more pronounced for a highly biased classifier. This is because empirically, for the Logistic Regression class, the auditor is better able to constrain the company's deviation with fairer audit sets. For higher $\Delta(h^\star)$, there is an increased scope of getting fairer audit sets with increasing $\alpha$.

\textbf{Effect of Initial Classifier Fairness (\( \Delta(h^\star) \)) on post-audit DP.} In this experiment, we fix the tolerance parameter $\alpha$ at $0.05$ and the audit budget at $25$ samples. We evaluate the effect of increasing \( \Delta(h^\star) \) on the DP violation of the post-audit model. Specifically, we compare the DP violation of the most biased classifier consistent with the audit sets produced by \textsf{RASC} and random selection. We also report the theoretical bounds corresponding to the optimal post-audit deviation achievable under the Budgeted and \( 0.05 \)-Tolerant Budgeted audits. 

As also observed for the Student Performance Dataset,  Figure~\ref{fig:combined_plots_compas} (Left) shows that the empirical DP violation of the company’s post-audit model increases with higher initial values of \( \Delta(h^\star) \). 

\noindent\textit{Comparison with Theoretical Behavior.} Similar to the Student Performance dataset, under the theoretical setting, the DP violation of the company’s post-audit model remains nearly constant across varying $\Delta(h^\star)$, suggesting a powerful company would be indifferent to initial fairness, while also achieving its objective of deviation almost perfectly due to the significantly lower size of the budget in comparison to the size of the whole dataset. 

\pgfplotscreateplotcyclelist{mycolorlist}{%
    {blue, mark=*},
    {red, mark=square*},
    {green!50!black, mark=triangle*},
    {purple!70!blue, mark=diamond*},
    {orange, mark=otimes*},
    {brown!60!black, mark=oplus*},
    {teal, mark=pentagon*},
    {magenta!60!black, mark=x}
}

\begin{figure*}[!t]
    \centering

    \begin{minipage}[b]{0.48\textwidth}
        \centering
        \begin{tikzpicture}
            \begin{axis}[
                title={$\Delta(h^\star) = 0.16$},
                xlabel={Budget},
                ylabel={$\Delta(h') - \Delta(h^\star)$},
                grid=minor,
                width = 5.5cm,
                cycle list name=mycolorlist,
            ]
            \addplot table[x index=0, y index=1, col sep = space]{varying_budget_fair_compas.txt};
            \addplot table[x index=0, y index=2, col sep = space]{varying_budget_fair_compas.txt};
            \addplot table[x index=0, y index=3, col sep = space]{varying_budget_fair_compas.txt};
            \addplot table[x index=0, y index=4, col sep = space]{varying_budget_fair_compas.txt};
            \addplot table[x index=0, y index=5, col sep = space]{varying_budget_fair_compas.txt};
            \addplot table[x index=0, y index=6, col sep = space]{varying_budget_fair_compas.txt};
            \addplot table[x index=0, y index=7, col sep = space]{varying_budget_fair_compas.txt};
            \addplot table[x index=0, y index=8, col sep = space]{varying_budget_fair_compas.txt};
            \end{axis}
        \end{tikzpicture}
    \end{minipage}
    \hfill
    \begin{minipage}[b]{0.48\textwidth}
        \centering
        \begin{tikzpicture}
            \begin{axis}[
                title={$\Delta(h^\star) = 0.50$},
                xlabel={Budget},
                ylabel={$\Delta(h') - \Delta(h^\star)$},
                grid=minor,
                width = 5.5cm,
                cycle list name=mycolorlist,
            ]
            \addplot table[x index=0, y index=1, col sep = space]{varying_budget_biased_compas.txt};
            \addplot table[x index=0, y index=2, col sep = space]{varying_budget_biased_compas.txt};
            \addplot table[x index=0, y index=3, col sep = space]{varying_budget_biased_compas.txt};
            \addplot table[x index=0, y index=4, col sep = space]{varying_budget_biased_compas.txt};
            \addplot table[x index=0, y index=5, col sep = space]{varying_budget_biased_compas.txt};
            \addplot table[x index=0, y index=6, col sep = space]{varying_budget_biased_compas.txt};
            \addplot table[x index=0, y index=7, col sep = space]{varying_budget_biased_compas.txt};
            \addplot table[x index=0, y index=8, col sep = space]{varying_budget_biased_compas.txt};
            \end{axis}
        \end{tikzpicture}
    \end{minipage}
    \begin{tikzpicture}
        \begin{axis}[
            hide axis,
            width=5cm,
            height=4cm,
            legend columns=2,
            legend style={
                draw=none,
                font=\footnotesize,
                row sep=-3pt,
                column sep=10pt
            },
            cycle list name=mycolorlist
        ]
        \addplot coordinates {(0,0)}; \addlegendentry{\textsf{RASC} with $\alpha = 0.05$}
        \addplot coordinates {(0,0)}; \addlegendentry{\textsf{RASC} with $\alpha = 0.1$}
        \addplot coordinates {(0,0)}; \addlegendentry{\textsf{RASC} with $\alpha = 0.2$}
        \addplot coordinates {(0,0)}; \addlegendentry{Random Selection}
        \addplot coordinates {(0,0)}; \addlegendentry{Budgeted Auditor}
        \addplot coordinates {(0,0)}; \addlegendentry{0.1-Tolerant Budgeted Auditor}
        \addplot coordinates {(0,0)}; \addlegendentry{0.2-Tolerant Budgeted Auditor}
        \addplot coordinates {(0,0)}; \addlegendentry{0.3-Tolerant Budgeted Auditor}
        \end{axis}
    \end{tikzpicture}

        \caption{Post-audit deviation with varying budget for $\Delta(h^\star) = 0.16$ (Left) and $\Delta(h^\star) = 0.50$ (Right).}
    \label{fig:combined_budget_plots_compas}
\end{figure*}
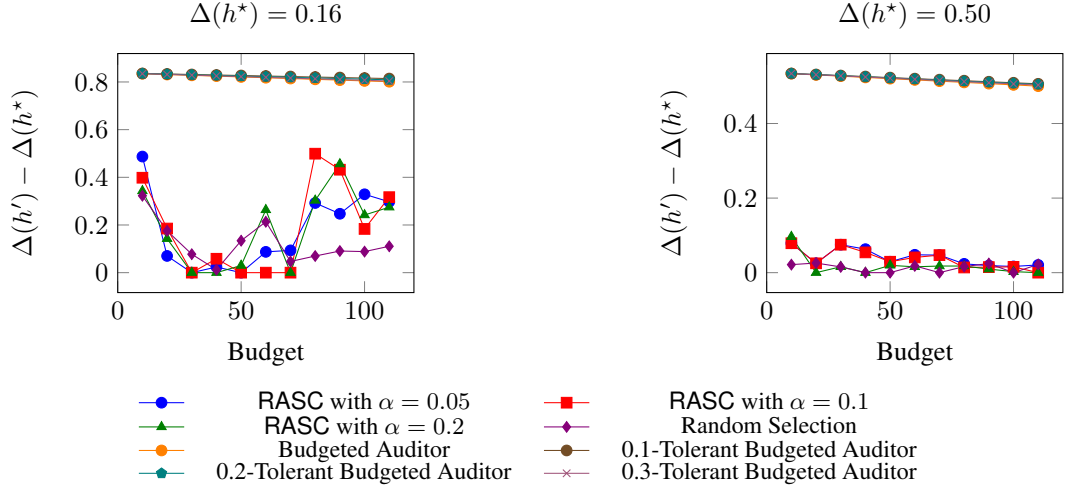

\pgfplotscreateplotcyclelist{mycolorlist}{
    {blue, mark=*},
    {red, mark=square*},
    {green!60!black, mark=triangle*},
    {orange, mark=diamond*},
}

\begin{figure*}
    
    \centering
    \begin{minipage}[b]{0.48\textwidth}
        \centering
        \begin{tikzpicture}
            \begin{axis}[
                title={$\Delta(h^\star) = 0.26$},
                xlabel={$\alpha$},
                ylabel={$\Delta(h') - \Delta(h^\star)$},
                grid=minor,
                width = 5.5cm,
                cycle list name=mycolorlist,
            ]
            \addplot table[x index=0, y index=1]{varying_alpha_fair_compas.txt}; 
            \addplot table[x index=0, y index=2]{varying_alpha_fair_compas.txt}; 
            \addplot table[x index=0, y index=3]{varying_alpha_fair_compas.txt}; 
            \addplot table[x index=0, y index=4]{varying_alpha_fair_compas.txt};
            \end{axis}
        \end{tikzpicture}
    \end{minipage}
    \hfill
    \begin{minipage}[b]{0.48\textwidth}
        \centering
        \begin{tikzpicture}
            \begin{axis}[
                title={$\Delta(h^\star) = 0.69$},
                xlabel={$\alpha$},
                ylabel={$\Delta(h') - \Delta(h^\star)$},
                grid=minor,
                width = 5.5cm,
                cycle list name=mycolorlist,
            ]
            \addplot table[x index=0, y index=1]{varying_alpha_biased_compas.txt};
            \addplot table[x index=0, y index=2]{varying_alpha_biased_compas.txt};
            \addplot table[x index=0, y index=3]{varying_alpha_biased_compas.txt};
            \addplot table[x index=0, y index=4]{varying_alpha_biased_compas.txt};
            \end{axis}
        \end{tikzpicture}
    \end{minipage}

    \begin{tikzpicture}
        \begin{axis}[
            hide axis,
            width=12cm,
            height=2cm,
            legend columns=3,
            legend style={
                draw=none,
                font=\small,
                row sep=-3pt,
                column sep=10pt
            },
            cycle list name=mycolorlist
        ]
        \addplot coordinates {(0,0)}; \addlegendentry{\textsf{RASC}}
        \addplot coordinates {(0,0)}; \addlegendentry{Random Selection}
        \addplot coordinates {(0,0)}; \addlegendentry{Budgeted Auditor}
        \addplot coordinates {(0,0)}; \addlegendentry{$\alpha$-Tolerant Budgeted Auditor}
        \end{axis}
    \end{tikzpicture}

    \caption{Post-audit deviation with varying $\alpha$ for  $\Delta(h^\star) = 0.26$ (Left) and $\Delta(h^\star) = 0.69$ (Right).}
    \label{fig:combined_alpha_plots_compas}
\end{figure*}
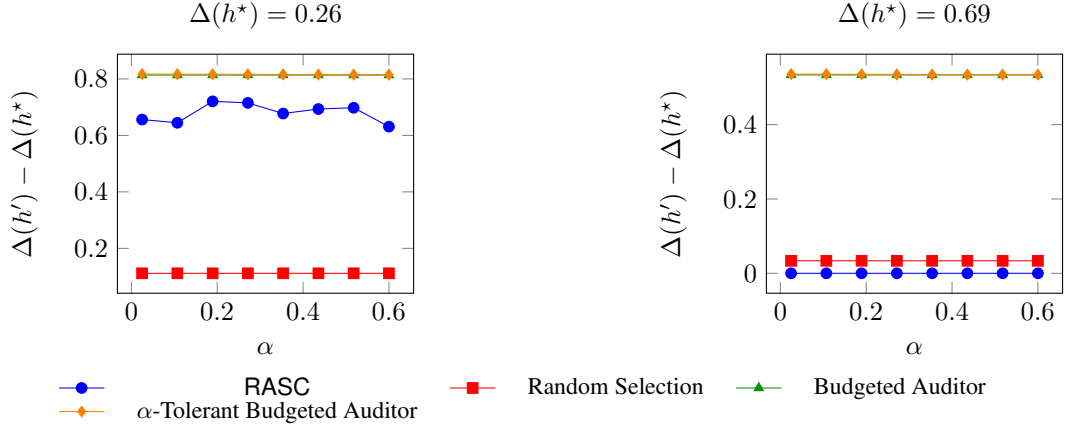

\pgfplotscreateplotcyclelist{mycolorlist}{
    {blue, mark=*},
    {red, mark=square*},
    {green!60!black, mark=triangle*},
    {orange, mark=diamond*},
}

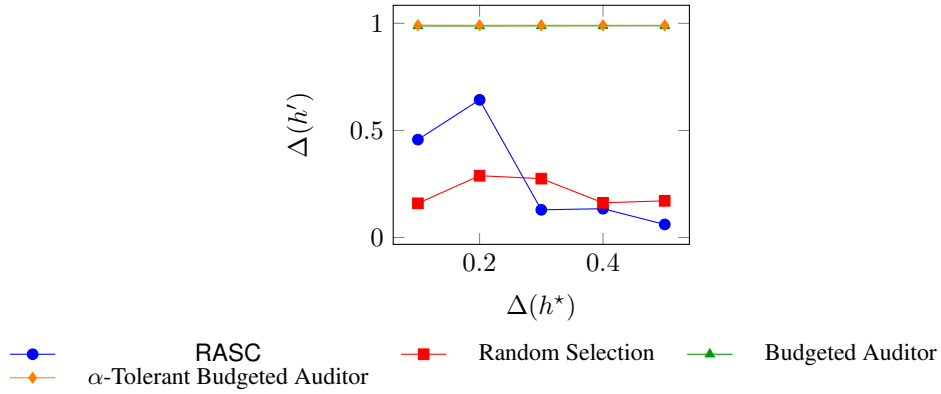
\begin{figure*}
    \centering

    \centering
    \begin{tikzpicture}
        \begin{axis}[
            xlabel={$\Delta(h^\star)$},
            ylabel={$\Delta(h')$},
            grid=minor,
            width = 5.5cm,
            cycle list name=mycolorlist,
        ]
        \addplot table[x index=0, y index=1]{varying_dp_compas.txt}; 
        \addplot table[x index=0, y index=2]{varying_dp_compas.txt}; 
        \addplot table[x index=0, y index=3]{varying_dp_compas.txt}; 
        \addplot table[x index=0, y index=4]{varying_dp_compas.txt}; 
        \end{axis}
    \end{tikzpicture}
   
    \begin{tikzpicture}
        \begin{axis}[
            hide axis,
            width=12cm,
            height=2cm,
            legend columns=3,
            legend style={
                draw=none,
                font=\small,
                row sep=-3pt,
                column sep=10pt
            },
            cycle list name=mycolorlist
        ]
        \addplot coordinates {(0,0)}; \addlegendentry{\textsf{RASC}}
        \addplot coordinates {(0,0)}; \addlegendentry{Random Selection}
        \addplot coordinates {(0,0)}; \addlegendentry{Budgeted Auditor}
        \addplot coordinates {(0,0)}; \addlegendentry{$\alpha$-Tolerant Budgeted Auditor}
        \end{axis}
    \end{tikzpicture}

    \caption{Post-audit DP violation with varying values of $\Delta(h^\star)$}
    \label{fig:combined_plots_compas}
\end{figure*}

\clearpage

\subsection{German Credit Risk Dataset}

\pgfplotscreateplotcyclelist{mycolorlist}{%
    {blue, mark=*},
    {red, mark=square*},
    {green!50!black, mark=triangle*},
    {purple!70!blue, mark=diamond*},
    {orange, mark=otimes*},
    {brown!60!black, mark=oplus*},
    {teal, mark=pentagon*},
    {magenta!60!black, mark=x}
}

\begin{figure*}[h!]
    \centering

    \begin{minipage}[b]{0.48\textwidth}
        \centering
        \begin{tikzpicture}
            \begin{axis}[
                title={$\Delta(h^\star) = 0.20$},
                xlabel={Budget},
                ylabel={$\Delta(h') - \Delta(h^\star)$},
                grid=minor,
                width = 5.5cm,
                cycle list name=mycolorlist,
            ]
            \addplot table[x index=0, y index=1, col sep = space]{varying_budget_fair_german.txt};
            \addplot table[x index=0, y index=2, col sep = space]{varying_budget_fair_german.txt};
            \addplot table[x index=0, y index=3, col sep = space]{varying_budget_fair_german.txt};
            \addplot table[x index=0, y index=4, col sep = space]{varying_budget_fair_german.txt};
            \addplot table[x index=0, y index=5, col sep = space]{varying_budget_fair_german.txt};
            \addplot table[x index=0, y index=6, col sep = space]{varying_budget_fair_german.txt};
            \addplot table[x index=0, y index=7, col sep = space]{varying_budget_fair_german.txt};
            \addplot table[x index=0, y index=8, col sep = space]{varying_budget_fair_german.txt};
            \end{axis}
        \end{tikzpicture}
    \end{minipage}
    \hfill
    \begin{minipage}[b]{0.48\textwidth}
        \centering
        \begin{tikzpicture}
            \begin{axis}[
                title={$\Delta(h^\star) = 0.50$},
                xlabel={Budget},
                ylabel={$\Delta(h') - \Delta(h^\star)$},
                grid=minor,
                width = 5.5cm,
                cycle list name=mycolorlist,
            ]
            \addplot table[x index=0, y index=1, col sep = space]{varying_budget_biased_german.txt};
            \addplot table[x index=0, y index=2, col sep = space]{varying_budget_biased_german.txt};
            \addplot table[x index=0, y index=3, col sep = space]{varying_budget_biased_german.txt};
            \addplot table[x index=0, y index=4, col sep = space]{varying_budget_biased_german.txt};
            \addplot table[x index=0, y index=5, col sep = space]{varying_budget_biased_german.txt};
            \addplot table[x index=0, y index=6, col sep = space]{varying_budget_biased_german.txt};
            \addplot table[x index=0, y index=7, col sep = space]{varying_budget_biased_german.txt};
            \addplot table[x index=0, y index=8, col sep = space]{varying_budget_biased_german.txt};
            \end{axis}
        \end{tikzpicture}
    \end{minipage}
    \begin{tikzpicture}
        \begin{axis}[
            hide axis,
            width=5cm,
            height=4cm,
            legend columns=2,
            legend style={
                draw=none,
                font=\footnotesize,
                row sep=-3pt,
                column sep=10pt
            },
            cycle list name=mycolorlist
        ]
        \addplot coordinates {(0,0)}; \addlegendentry{\textsf{RASC} with $\alpha = 0.05$}
        \addplot coordinates {(0,0)}; \addlegendentry{\textsf{RASC} with $\alpha = 0.1$}
        \addplot coordinates {(0,0)}; \addlegendentry{\textsf{RASC} with $\alpha = 0.2$}
        \addplot coordinates {(0,0)}; \addlegendentry{Random Selection}
        \addplot coordinates {(0,0)}; \addlegendentry{Budgeted Auditor}
        \addplot coordinates {(0,0)}; \addlegendentry{0.1-Tolerant Budgeted Auditor}
        \addplot coordinates {(0,0)}; \addlegendentry{0.2-Tolerant Budgeted Auditor}
        \addplot coordinates {(0,0)}; \addlegendentry{0.3-Tolerant Budgeted Auditor}
        \end{axis}
    \end{tikzpicture}

        \caption{Post-audit deviation with varying budget for $\Delta(h^\star) = 0.20$ (Left) and $\Delta(h^\star) = 0.50$ (Right).}
    \label{fig:combined_budget_plots_german}
\end{figure*}
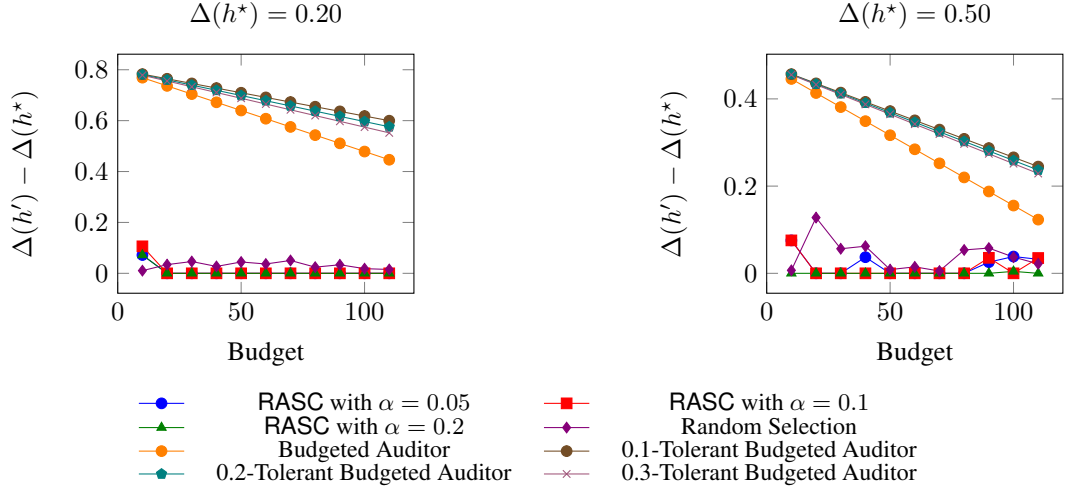

\pgfplotscreateplotcyclelist{mycolorlist}{
    {blue, mark=*},
    {red, mark=square*},
    {green!60!black, mark=triangle*},
    {orange, mark=diamond*},
}

\begin{figure*}[h!]
    \centering
    \begin{minipage}[b]{0.48\textwidth}
        \centering
        \begin{tikzpicture}
            \begin{axis}[
                title={$\Delta(h^\star) = 0.28$},
                xlabel={$\alpha$},
                ylabel={$\Delta(h') - \Delta(h^\star)$},
                grid=minor,
                width = 5.5cm,
                cycle list name=mycolorlist,
            ]
            \addplot table[x index=0, y index=1]{varying_alpha_fair_german.txt}; 
            \addplot table[x index=0, y index=2]{varying_alpha_fair_german.txt}; 
            \addplot table[x index=0, y index=3]{varying_alpha_fair_german.txt}; 
            \addplot table[x index=0, y index=4]{varying_alpha_fair_german.txt};
            \end{axis}
        \end{tikzpicture}
    \end{minipage}
    \hfill
    \begin{minipage}[b]{0.48\textwidth}
        \centering
        \begin{tikzpicture}
            \begin{axis}[
                title={$\Delta(h^\star) = 0.83$},
                xlabel={$\alpha$},
                ylabel={$\Delta(h') - \Delta(h^\star)$},
                grid=minor,
                width = 5.5cm,
                cycle list name=mycolorlist,
            ]
            \addplot table[x index=0, y index=1]{varying_alpha_biased_german.txt};
            \addplot table[x index=0, y index=2]{varying_alpha_biased_german.txt};
            \addplot table[x index=0, y index=3]{varying_alpha_biased_german.txt};
            \addplot table[x index=0, y index=4]{varying_alpha_biased_german.txt};
            \end{axis}
        \end{tikzpicture}
    \end{minipage}

    \begin{tikzpicture}
        \begin{axis}[
            hide axis,
            width=12cm,
            height=2cm,
            legend columns=3,
            legend style={
                draw=none,
                font=\small,
                row sep=-3pt,
                column sep=10pt
            },
            cycle list name=mycolorlist
        ]
        \addplot coordinates {(0,0)}; \addlegendentry{\textsf{RASC}}
        \addplot coordinates {(0,0)}; \addlegendentry{Random Selection}
        \addplot coordinates {(0,0)}; \addlegendentry{Budgeted Auditor}
        \addplot coordinates {(0,0)}; \addlegendentry{$\alpha$-Tolerant Budgeted Auditor}
        \end{axis}
    \end{tikzpicture}

    \caption{Post-audit deviation with varying $\alpha$ for  $\Delta(h^\star) = 0.28$ (Left) and $\Delta(h^\star) = 0.83$ (Right).}
    \label{fig:combined_alpha_plots_german}
\end{figure*}
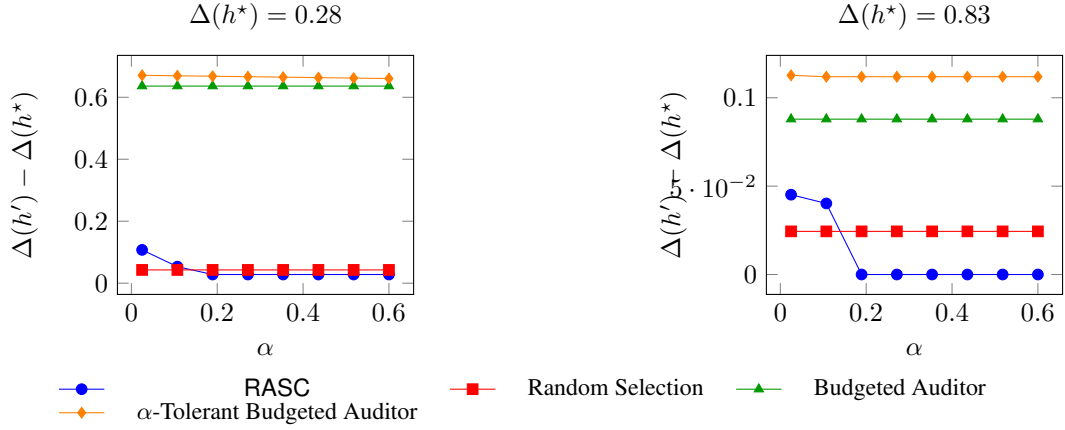

\pgfplotscreateplotcyclelist{mycolorlist}{
    {blue, mark=*},
    {red, mark=square*},
    {green!60!black, mark=triangle*},
    {orange, mark=diamond*},
}

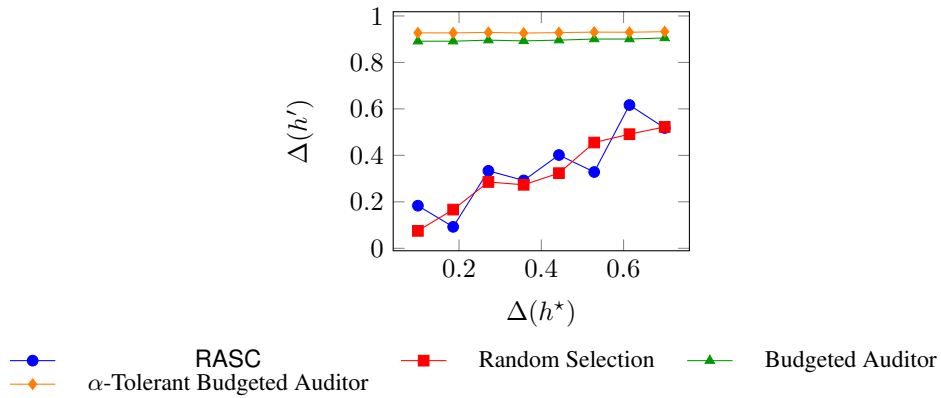
\begin{figure*}[h!]
    \centering

        \centering
        \begin{tikzpicture}
            \begin{axis}[
                xlabel={$\Delta(h^\star)$},
                ylabel={$\Delta(h')$},
                grid=minor,
                width = 5.5cm,
                cycle list name=mycolorlist,
            ]
            \addplot table[x index=0, y index=1]{varying_dp_german.txt}; 
            \addplot table[x index=0, y index=2]{varying_dp_german.txt}; 
            \addplot table[x index=0, y index=3]{varying_dp_german.txt}; 
            \addplot table[x index=0, y index=4]{varying_dp_german.txt}; 
            \end{axis}
        \end{tikzpicture}

    \begin{tikzpicture}
        \begin{axis}[
            hide axis,
            width=12cm,
            height=2cm,
            legend columns=3,
            legend style={
                draw=none,
                font=\small,
                row sep=-3pt,
                column sep=10pt
            },
            cycle list name=mycolorlist
        ]
        \addplot coordinates {(0,0)}; \addlegendentry{\textsf{RASC}}
        \addplot coordinates {(0,0)}; \addlegendentry{Random Selection}
        \addplot coordinates {(0,0)}; \addlegendentry{Budgeted Auditor}
        \addplot coordinates {(0,0)}; \addlegendentry{$\alpha$-Tolerant Budgeted Auditor}
        \end{axis}
    \end{tikzpicture}

    \caption{Post-audit DP violation with varying values of $\Delta(h^\star)$} 
    \label{fig:combined_plots_german}
\end{figure*}

\end{document}